%% file: main.tex
\documentclass[letterpaper, 10 pt, conference]{ieeeconf}  
\IEEEoverridecommandlockouts
\usepackage{graphicx}
\usepackage{epstopdf}
\usepackage{amsmath}
\usepackage{amssymb}
\usepackage{subfigure}
\usepackage{multirow}
\usepackage{pbox}
\usepackage{cite}
\usepackage{algorithm}
\usepackage{algpseudocode}
\usepackage{booktabs}
\usepackage{wrapfig}
\usepackage{lscape}
\usepackage{bm}
\usepackage{url}
\usepackage{esvect}
\usepackage{xcolor}
\usepackage{txfonts}
\usepackage[normalem]{ulem}
\usepackage{easyReview}

\DeclareSymbolFont{extraup}{U}{zavm}{m}{n}
\DeclareMathSymbol{\varheart}{\mathalpha}{extraup}{86}
\DeclareMathSymbol{\vardiamond}{\mathalpha}{extraup}{87}

\newcommand{\R}{\mathbb{R}}

\newcommand{\N}{\mathbb {N}}

\algnewcommand\algorithmicforeach{\textbf{for each}}
\algdef{S}[FOR]{ForEach}[1]{\algorithmicforeach\ #1\ \algorithmicdo}

\newcommand{\track}{\text{{\fontfamily{qcr}{\selectfont track}}}}
\newcommand{\anchor}{\text{{\fontfamily{qcr}{\selectfont anchor}}}}
\newcommand\Bigger[2][7]{\left#2\rule{0mm}{#1truemm}\right.}

\newtheorem{lemma}{Lemma}

\renewcommand{\baselinestretch}{0.94}

\begin{document}

\title{\LARGE \bf Epistemic Prediction and Planning with Implicit Coordination for Multi-Robot Teams in Communication Restricted Environments}
\author{Lauren Bramblett, Shijie Gao, and Nicola Bezzo%
\thanks{Lauren Bramblett, Shijie Gao, and Nicola Bezzo are with the Departments of Engineering Systems and Environment and Electrical and Computer Engineering, University of Virginia, Charlottesville, VA 22904, USA. Email: {\tt \{qbr5kx sg9dn nb6be\}@virginia.edu}}}

\maketitle

\setreviewson
\input{0Abstract}
\input{1Introduction}
\input{2RelatedWork}
\input{3Preliminary}

\input{4Problem}

\input{5Approach}
\input{6Results}
\input{7Conclusion}

\newpage
\bibliographystyle{IEEEtran}
\bibliography{newRef.bib}

\end{document}

%% file: 0Abstract.tex
\begin{abstract} 
In communication restricted environments, a multi-robot system can be deployed to either: i) maintain constant communication but potentially sacrifice operational efficiency due to proximity constraints or ii) allow disconnections to increase environmental coverage efficiency, challenges on how, when, and where to reconnect (rendezvous problem). In this work we tackle the latter problem and notice that most state-of-the-art methods assume that robots will be able to execute a predetermined plan; however system failures and changes in environmental conditions can cause the robots to deviate from the plan with cascading effects across the multi-robot system. This paper proposes a coordinated epistemic prediction and planning framework to achieve consensus without communicating for exploration and coverage, task discovery and completion, and rendezvous applications. Dynamic epistemic logic is the principal component implemented to allow robots to propagate belief states and empathize with other agents. Propagation of belief states and subsequent coverage of the environment is achieved via a frontier-based method within an artificial physics-based framework. The proposed framework is validated with both simulations and experiments with unmanned ground vehicles in various cluttered environments. 

\vspace{2pt}
 
\end{abstract}

%% file: 1Introduction.tex
\section{Introduction}
Multi-robot systems (MRS) have the potential to assist in many safety-critical applications such as search and rescue, military intelligence and surveillance, and inspection operations where it may be hazardous and costly to deploy humans. Looking to the state-of-the-art, we note that most MRS research assumes constant communication between robots \cite{capelli2020connectivity,hussein2014multi,morilla2022sweep}. However, within the aforementioned application space, long-range communication is often unreliable or unavailable. Humans adequately cope with such problems, performing these tasks collaboratively by extrapolating and empathizing with what other actors might believe if the local plan must change at run-time. This subconscious process can be modally represented as epistemic planning, computing and reasoning about multiple predictions and actions while accounting for a priori beliefs, current observations,
and other actors' sensing and mobility capabilities.
  
In this work, we insist that if the robots in a team could perform similar reasoning without communication then we could relax the typical connectivity constraints, while increasing autonomy (i.e., decrease human intervention) and mission performance (i.e., more coverage, faster task discovery and completion). For this reason, we propose a novel epistemic planning framework for multi-robot systems, allowing each robot to cooperate without constant communication by reasoning about teammates' performance and states, through the propagation of beliefs states about each others. 
\begin{figure}
    \includegraphics[width = 0.48\textwidth]{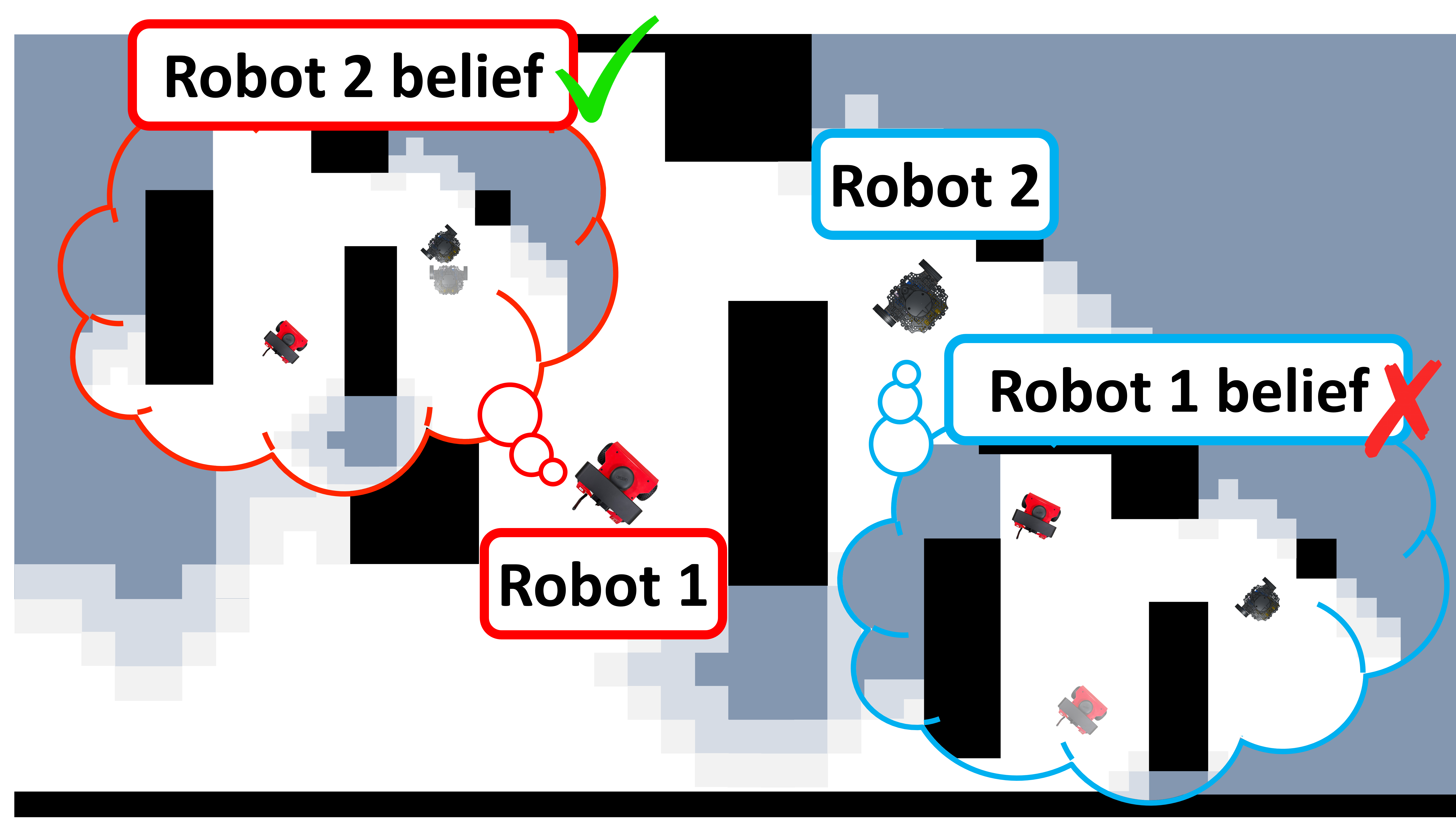}
    \vspace{-5pt}
    \caption{Pictorial depiction of the problem. The proposed framework enables a robot to reason from other agents' perspectives as it experiences a behavior change or observes that another robot is not where expected.}  
    \label{fig:introPic}
    \vspace{-18pt}
\end{figure}
As a reader can note, calculating a distributed plan for coverage while accounting for any combination of robot system failures, changes in the environments, or deviations is intractable. Instead, constructing a finite set of possibilities and implementing a reasoning framework for each robot can reduce computational complexity and allow for more efficient operations. Thus, we introduce a coordinated epistemic prediction and planning method in which a robot propagates a finite set of \textit{belief} states representing possible states of other agents in the system and \textit{empathy} states representing a finite set of possible states from other agents' perspectives. Subsequently, using epistemic planning, we can formulate a consensus strategy such that every distributed belief in the system achieves consensus. For example, consider Fig.~\ref{fig:introPic} where two robots are canvassing an environment.  During disconnection, Robot 1 maintains a set of possible (belief) states for Robot 2 and also a set of (empathy) states that Robot 2 might believe about Robot 1. Once Robot 2 experiences a failure, it tracks another state in its empathy set. We reason that though Robot 1 holds a false belief about Robot 2's state, there exists an epistemic strategy that can allow robot 1 to find robot 2 (i.e., updating its belief after observing robot 1's believed state).

The contribution of our approach is two-fold: i) an epistemic planning formulation using dynamic epistemic logic, formalizing beliefs and knowledge for robot control and ii) a generalized task assignment and artificial potential field-based model for belief propagation and coverage of an environment with considerations for connectivity constraints and team member dynamics.

%% file: 2RelatedWork.tex
\section{Related Work}
\label{sec:relatedwork}
Heterogeneous multi-robot exploration, foraging, and coverage have been widely studied in robotics literature \cite{Lin2018,zhou2019bayesian,reily2021adaptation}. Recent works consider communication restricted or intermittent connectivity by modeling ways to maintain connectivity while exploring \cite{capelli2020connectivity}, to account for momentary disconnection \cite{best2018planning}, or dropping intermediary notifications to other robots \cite{cardona2019ant}. Authors in \cite{matignon2012coordinated} use a decentralized Markov decision process to estimate future locations of robots when inter-robot messages are delayed at a stochastic rate. Several works have also included system failures or disturbances in multi-robot exploration policies, such as in the work \cite{al2018generation}; however, the policy assumes robots are able to communicate these disruptions. In most approaches, the execution policy is static and communication is a constraint that either needs to be satisfied for all time or at a defined location \cite{cesare2015multi,el2020hdec}.

A separate, but related, field to multi-robot coverage is multi-robot task allocation (MRTA). MRTA is solved by uniquely assigning a subset of robots to optimize the completion of an objective \cite{kim2020multiplicatively}. Authors in \cite{chen2022consensus} include connection limitations and allocate tasks to connected robots using a bundling algorithm. Another example in \cite{khodayi2019distributed} allocated targets to individual teams and plans rendezvous with team members to reduce uncertainty of targets over time.  

Though recent works in multi-robot task allocation and coverage have included realistic constraints, there is little consideration for the combination of prolonged disconnection and system failures. In our previous work \cite{Bramblett2022}, we define rendezvous points at known locations to coordinate roles for any events during exploration; however, back-tracking to this location reduced efficiency of exploration. In contrast, this work applies dynamic epistemic logic (DEL) \cite{van2007dynamic} to allow a robot to reason about beliefs among actors in a multi-robot system while disconnected and converge to a dynamic rendezvous location. While classically used to describe how knowledge and information changes for players in a game, DEL is not a new concept in robotics. Using robot and human actors, the framework in \cite{bolander2021based} recreates the Sally-Anne psychological test where a robot must reason about the human's beliefs. Typical multi-robot applications use DEL to solve for a cooperative set of actions in multi-player games \cite{maubert2021concurrent}. We extend DEL for a realistic multi-robot application allowing each robot to reason about the system's state considering system failures, task discovery, and partially known environments. 

%% file: 3Preliminary.tex
\section{Preliminaries}
\label{sec:preliminaries}
\subsection{Notation \& Communication}
Let us consider a  multi-robot system of $N_a$ robots in the set $\mathcal{A}$, noting that our approach is suitable for heterogeneous robots of differing capabilities (e.g., dynamics, sensing, etc.). The system's connectivity graph is denoted as $G=(\mathcal{A},\mathcal{E})$ where the set $\mathcal{E} \subset \mathcal{A} \times \mathcal{A}$ represents edge connections between robots and an edge $(i,j) \in \mathcal{E}$ indicates that robots $i$ and $j$ are connected.
Additionally, $N_t$ tasks in the set $\mathcal{T}$ are located in unknown positions within the operating environment. We assume the tasks are stationary and completed once any robot $i \in \mathcal{A}$ navigates within a radius $r_t > 0$, considering that the number of tasks $N_t$ is initially unknown to the robots. 
The robots are assigned to search for the tasks in an environment that is partitioned into $N_m$ cells, which we define as an occupancy map $\mathcal{M}\subseteq \R^2$. When robots navigate to observe unexplored cells $\mathcal{M}_u \subseteq \mathcal{M}$, $\mathcal{M}$ is updated using Recursive Bayesian estimation \cite{asgharivaskasi2021active}, though any method can be used. Subsequently, we define the frontier set $\mathcal{F}$ as the set of explored cells adjacent to unknown cells.

\subsection{Epistemic Logic}
\label{sec:epiLog}
In this work, epistemic and doxastic logic \cite{fagin2004reasoning} is used to model distributed knowledge and reasoning for non-catastrophic system changes during disconnectivity. To propagate uncertain states of robots, we define $N_b$ beliefs in the set $\mathcal{B}$ as the behaviors likely to occur to agents during operation. The set $\mathcal{P} = \{\mathcal{P}_1 , \dots , \mathcal{P}_{N_a} \}$ holds the distributed beliefs of all agents, where an element in $\mathcal{P}_i$ represents possible states from an agent $i$'s perspective of robots $j \in \mathcal{A}$. $\Psi$ is a set of functions that describe the current state of the system. For this application, the epistemic language, $\mathcal{L}(\Psi,\mathcal{P},\mathcal{A}$), is obtained as follows in Backus-Naur form \cite{knuth1964backus}:
\begin{equation*}
    \phi \Coloneqq H(\omega) \ | \ i\sphericalangle j\ | \ \phi\land\phi \ | \ \neg\phi \ | \ K_i\phi \ | \ B_i\phi  
\end{equation*}
where  $i,j\in\mathcal{A}$, $H\in\Psi$ is a function to describe a system state, and $\omega$ broadly indicates function arguments. $\neg\phi$ and $\phi\land\phi$ denote that propositions can be negated and form logical conjunctions. $B_i\phi$ and $K_i\phi$ are interpreted as ``agent $i$ believes $\phi$" and ``agent $i$ knows $\phi$." $i\sphericalangle j$ is an observability atom that reads ``agent $i$ is within communication range of agent $j$". 

A pointed Kripke model ($M$) represents an epistemic model \cite{browne1988characterizing} consisting of the set of possible states of the system referred to as worlds, accessibility relations ($R$) between worlds, $w$, that are possible for an agent, and a valuation ($V$) that labels true propositions in each world.
Dynamic epistemic logic formalizes Kripke model transforms given an event, $a$. In this work, we use dynamic epistemic logic and Kripke models to modally represent a distributed strategy using logical belief-based reasoning. 

%% file: 4Problem.tex
\section{Problem Formulation} \label{sec:probform}
In this paper, we consider a scenario in which a multi-robot system must coordinate in a decentralized fashion to efficiently search for tasks at unknown locations in a communication restricted, partially-known environment. There are several challenges that arise to allow efficient and cooperative behavior given limited communication including: 1) how to efficiently explore, search, and cover a partially known environment while remaining disconnected for extended periods of time and 2) how to properly plan to take into account uncertainties that could happen during disconnection that lead to different behaviors of the robots in the systems.
Formally, we define this problem as:

\textbf{Problem 1 (\textit{Communication restricted coverage}):} Find a distributed policy to enable a multi-robot system to quickly perform cooperative search of an environment for tasks with intermittent communication accounting for uncertainties. The policy should consider minimizing mission time while enabling periodic information sharing to cooperatively allocate portions of the environment if a robot's behavior changes.

%% file: 5Approach.tex
\section{Approach}
\label{sec:approach}
In this section, we present the approach for the coordinated epistemic prediction and planning framework which propagates belief and empathy states to inform frontier assignment and robot control, all while considering failures, task discovery, and unknown obstacles. For ease of discussion let us consider two robots $i$ and $j$. From robot $i$'s perspective, a {\em{belief state}}, $p_{ij,b}\in\mathcal{P}_i$, represents a possible state of a robot $j$ and an {\em{empathy state}}, $p_{ii,b}\in\mathcal{P}_i$, describes robot $i$'s belief of robot $j$'s belief about robot $i$'s state. With this knowledge, robot $i$ predicts and tracks empathy states to ensure that a robot $j$ holds one true belief of the state of robot $i$. The diagram in Fig.~\ref{fig:mainFrame} summarizes this architecture.

\begin{figure}[t]
\vspace{-1pt}
    \includegraphics[width = 0.48\textwidth]{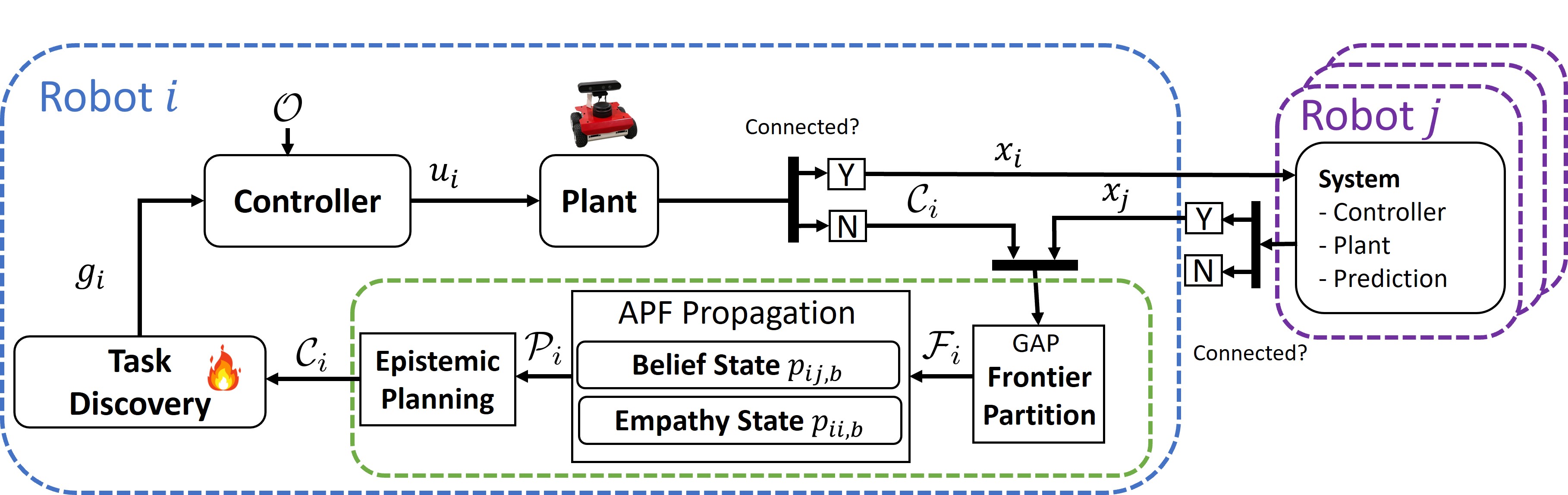}
    \caption{Diagram of the proposed approach. The contributions of this paper are within the green box.}
    \label{fig:mainFrame}
    \vspace{-15pt}
\end{figure}

As shown in Fig.~\ref{fig:mainFrame}, the robot $i$ initially assesses whether communication is successful with a robot $j$. If communication is successful, robot $i$ uses its current state $x_i$ and the state of robot $j$, $x_j$, to partition its frontiers using a generalized assignment problem (GAP) \cite{oncan2007survey} and to predict future states of robot $j$ using an APF method. When connected, epistemic planning is reduced to direct communication of states. If the robots disconnect, a common belief set, $\mathcal{C}_i$, acts as the state for any robot $j \in \mathcal{A}$ from $i$'s perspective.
Predictions for these belief and empathy states are accomplished using the same GAP and APF methods. A robot $i$ then uses these predicted states to plan considering its belief about robot $j$. 

In both connected and disconnected conditions, the robot's objective is to search for tasks. If connected and a task is discovered, the robots bid on and accomplish tasks. If disconnected, the robots will deviate to accomplish the task and subsequently continue to track its empathy state.

In the following sections, we lay out the key components of the planner including: i) belief propagation, ii) MRS coverage assignments, iii) epistemic planning for belief consensus, and iv) obstacle avoidance and task discovery. 

\subsection{Belief \& Empathy Propagation} 
In our coordinated epistemic prediction and planning framework, the robots propagate belief and empathy states for all robots in the multi-robot system. This allows a robot $i$ to plan according to its belief of other robots and reason about what other robots' expect robot $i$ to accomplish while disconnected. As previously noted, to account for uncertainties over long periods of disconnection, it is important to have a finite number of these states. With this goal in mind, we define a finite set of particles, $\mathcal{P}_i$, to represent these belief and empathy states for the $i^{th}$ robot: 
\begin{equation}
    \mathcal{P}_i=\{p_{ij,b} \ \forall j\in\mathcal{A},\forall b\in\mathcal{B}\}.
\end{equation}
The $i^{th}$ robot defines its empathy particles as $\mathcal{P}^e_{i} = \{p_{ii,b} \ \forall b\in \mathcal{B}\}$ and its belief particles about other robots as $\mathcal{P}^r_{i} = \{p_{ij,b}  \ \forall j\in \mathcal{A}\setminus\{i\},\forall b\in \mathcal{B}\}$
where $\mathcal{P}_i = \mathcal{P}_i^e \cup \mathcal{P}_i^r$.
For each robot $j \in \mathcal{A}$, the robot $i$ orders its belief and empathy particles $1$ through $N_b$ by likelihood of occurrence (i.e., from largest to smallest). The order is initialized prior to deployment and each robot $i$ initially tracks its first empathy particles.

While not in communication range of other robots, each robot $i$ has a common belief about each robot $j$ and itself. We define robot $i$'s common belief as $\mathcal{C}_i \subseteq \mathcal{P}_i$ and refer to it as the \textit{common belief set}. All robots track their first empathy particle upon disconnection, $\mathcal{C}_i = \{p_{ij,1} \ \forall j\in\mathcal{A}\}$.

If a robot experiences a failure, choosing which of its next empathy particles to track is nontrivial, but we assume each robot $i$ is capable of computing the set of empathy states that are suitable to track, denoted by $\mathcal{P}^t_{i}\subseteq\mathcal{P}^e_{i}$. The robot chooses to track the particle in $\mathcal{P}^t_{i}$ with the highest likelihood. If all robots are within communication range, the first particle becomes the robot's current state and subsequent particles are propagated based on the updated common belief.

Since the robot will be tracking an empathy particle, these states must propagate in a manner that allows the robot to safely and efficiently accomplish the coverage objective with considerations for intentional information sharing. Thus, we propagate the particles using an artificial potential field (APF) that leverages four main objectives: 1) attraction to frontier, 2) cooperative rendezvous, 3) obstacle avoidance, and 4) task completion. In this propagation method, the total force acting on particle $p_{ij,b}$ is formulated generally as: 
\begin{equation}
    \vspace{-2pt}
    F_{ij,b}^{total} = \beta_1 F_{ij,b}^1 + \beta_2 F_{ij,b}^2 + \beta_3 F_{ij,b}^3 + \beta_4 F_{ij,b}^4 
    \label{eq:taskTotalForce}
        \vspace{-2pt}
\end{equation}
considering $\beta_n$ is a weighting coefficient for force $F^{n}_{ij,b}$ where each $F^n_{ij,b}$ corresponds to the $n^{th}$ objective listed previously and will be discussed in detail. Local minima is avoided using an A$^*$ path planner \cite{julia2008local}. 
\vspace{-0pt}
\subsection{Frontier Attraction}
A frontier-based exploration method is proposed here due to its completeness and simplicity. To begin, the force $F^1_{ij,b}$ in \eqref{eq:taskTotalForce} is an attraction to a frontier set $\mathcal{F} \subset \mathcal{M}$. 
However, a robot should only traverse unique portions of the environment to reduce redundancy and minimize completion time. So, a decentralized GAP assigns particles to a unique subset of $\mathcal{F}$ using its belief of each robots' capabilities.


For particle $p_{ij,b}$, we allocate frontiers based on the cost of assigning the particle or any other robot. Cost and binary assignment are denoted as $\Lambda$ and $\Gamma$, respectively. The corresponding GAP is formulated as:
\begin{align}
    \vspace{-5pt}
\mathcal{F}_{ij,b} =  \min &\sum_{k\in\mathcal{A}}\sum_{z\in \mathcal{F}} \lambda_{zk}\gamma_{zk}\nonumber\\
\text{s.t.} &\sum_{k\in\mathcal{A}} \gamma_{zk}=1, \ \forall z\in \mathcal{F}\nonumber\\
&\sum_{z\in \mathcal{F}}\gamma_{zk}\leq u, \ \forall k \in \mathcal{A} \label{eq:gapProb}\\
&\sum_{z\in \mathcal{F}}\gamma_{zk}\geq \ell, \ \forall k \in \mathcal{A} \nonumber\\
&\gamma_{zk}\in \{0,1\}, \ \forall k\in \mathcal{A}, \ \forall z \in \mathcal{F} \nonumber
    \vspace{-5pt}
\end{align}
where the elements $\lambda_{zk} \in \R_{\geq 0}$ and $\gamma_{zk}$ represent the $z^{th}$ frontier and $k^{th}$ particle in the matrices $\Lambda$ and $\Gamma$, respectively. Cost generally refers to any traversal metric (i.e., energy, time, etc.) to a frontier point $z \in\mathcal{F}$. The variables $u \in \N$ and $\ell \in \N$ are the upper and lower bounds on the number of frontier points that can be assigned to any particle. Each robot $i$ calculates the frontier assignment for each particle in the set $\mathcal{P}_i$ and the assigned frontier set is denoted as $\mathcal{F}_{{ij,b}}$. 

Subsequently, to utilize the GAP solution, the force $F^1_{ij,b}$ controls the $b^{th}$ particle $p_{ij,b}$ towards its assigned frontiers
\begin{equation}
    F^1_{ij,b} = \dfrac{1}{|\mathcal{F}_{ij,b}|}\sum_{z\in \mathcal{F}_{{ij,b}}}\dfrac{s_z-p_{ij,b}}{||s_z-p_{ij,b}||^3} \label{eq:gotofront}
\end{equation}
where $|\cdot|$ indicates the set's cardinality and the coordinate of a $z^{th}$ frontier is designated as $s_z$. The force computed in \eqref{eq:gotofront} encourages particle motion to the unexplored regions of the environment $\mathcal{M}_u$ based on their frontier assignment in $\eqref{eq:gapProb}$. 
\subsection{Epistemic Planning}

Intentional information sharing allows an agent to communicate any environmental or capability changes with other robots. For this purpose, we introduce $F^2_{ij,b}$ in \eqref{eq:taskTotalForce} to control each particle $p_{ij,b}$ according to robot $i$'s common belief set:
\begin{align}
F^2_{ij,b} &= \varphi_i\sum_{k\in\mathcal{A}} \dfrac{c_k-p_{ij,b}}{||c_k-p_{ij,b}||^3}\label{eq:cooperativeBehavior} \\
\varphi_i &= \Bigger[5]\{\begin{array}{@{}rl}
    -h(t_r,\tau), & t_r<\tau\\
    h(t_r,\tau), & t_r\geq \tau
\end{array}
\end{align}
where $c_k$ denotes the $k^{th}$ element in $\mathcal{C}_i$. Variable $\tau$ is a time-based threshold for rendezvous and $h:(t_r,\tau)\mapsto\R_{\geq 0}$ where $t_r$ is the time lapsed since the last successful communication. 

Fig.~\ref{fig:exampleAPF} shows the effect of $\varphi_{i}$. Given the partitioned frontier from \eqref{eq:gotofront}, the robots' particles are incentivized to travel i) away from $c_k$ when $\varphi_{i}<0$, ii) towards its assigned frontier when $\varphi_{i} = 0$, or iii) towards $c_k$ when $\varphi_{i}>0$. We denote the line between common belief particles as the \textit{anchor line}.
\begin{figure}[b]
\vspace{-9pt}
    \centering
    \subfigure[$\varphi_i<0$]{\includegraphics[width=0.18\textwidth]{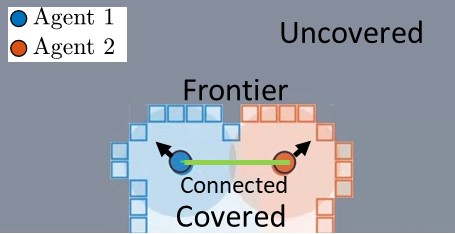}}
    \subfigure[$\varphi_i=0$]{\includegraphics[width = 0.18\textwidth]{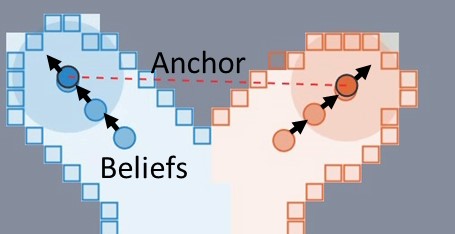}}
    \subfigure[$\varphi_i>0$]{\includegraphics[width = 0.18\textwidth]{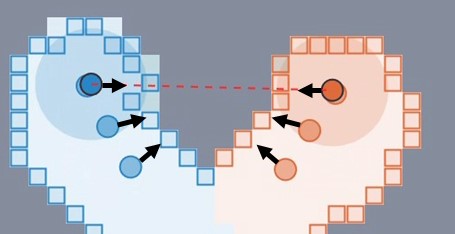}}
    \subfigure[Reconnected and reset]{\includegraphics[width = 0.18\textwidth]{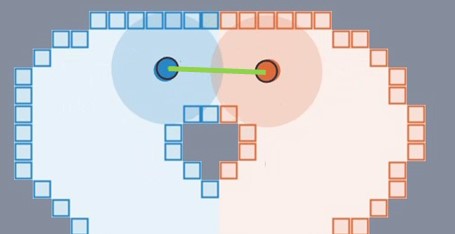}}
    \vspace{-5pt}
    \caption{Depiction of APF forces for each particle given different $\varphi_i$ over time. The color of the anchor line indicates communication (green) or no communication (red).}
    \label{fig:exampleAPF}
    \vspace{-5pt}
\end{figure}
In this way, \eqref{eq:cooperativeBehavior} controls all of robot $i$'s particles to all beliefs in $\mathcal{C}_i$ when $t_r>\tau$. This is formalized in the following lemma:
\begin{lemma}
If $F^2_{ij,b} = 0$, as $t\rightarrow\infty$ all particles in the set $\{p_{ij,b} \ | \ p_{ij,b}\in \mathcal{P}\}$ will converge to $c_k, \ \forall k\in\mathcal{A}$.
\end{lemma}
\begin{proof}
Given the $\lim_{t\rightarrow\infty}F^1_{ij,b} + F^4_{ij,b} = 0$ once all area has been covered, the only force acting on each particle will be $F^2_{ij,b}$. Also, $\varphi_i>0$ since $t_r>\tau$ when $t\rightarrow\infty$, all particles in the set $\{p_{ij,b} \ | \ p_{ij,b}\in \mathcal{P}_i\}$ will converge to $c_k\in\mathcal{C}_i$.
\end{proof}

Considering that the $c_k\in \mathcal{P}_i$ is controlled via \eqref{eq:cooperativeBehavior}, all common belief states in $\mathcal{C}_i$ converge and so, all particles in $\mathcal{P}_i$ converge to the same rendezvous location. Thus, it is imperative to ensure that while the robots are not communicating they can reach consensus such that
 \begin{equation}
     \mathcal{C}_i\equiv \mathcal{C}_j, \ \ \forall i,j\in \mathcal{A}.
 \end{equation}
We coordinate this consensus using dynamic epistemic logic. 

Referring to the previously established semantics for DEL in Sec.~\ref{sec:preliminaries}, we introduce the set $\Psi$ consisting of a binary and tertiary function, \track ~and \anchor,  noting the argument for the $i^{th}$ robot is denoted as $A_i$ for readability in the epistemic model. The function $\track(A_i,p_{ii,b})$ is read as ``robot $i$ is tracking empathy particle $b$" and $\anchor(A_i,p_{ii,b},p_{ij,b})$ is read as ``robot $i$ is using belief particle $p_{ij,b}$ and empathy particle $p_{ii,b}$ as its common belief." 

Since a robot's failures can affect a robot's capabilities in disparate ways and at any time, we formulate a strategy for rendezvous that accounts for all possible combinations of failures. Given that robots are connected given a communication graph $G$, we also assume particles can observe each other similarly (e.g., based on range, line-of-sight), denoted logically as $p_{ij,b}\sphericalangle p_{ik,b}$. If two common belief particles have observed each other, we know that one of four events have occurred for two robots: i) $A_i$ has observed that $A_j$ is not tracking the common belief particle, ii) $A_j$ has observed that $A_i$ is not tracking the common belief particle, iii) neither agents observe either common belief particle, or iv) both agents communicate. In events (i)-(iii), neither robot knows the true system state since the robots did not communicate. 

Given that our current common belief is the $b^{\text{th}}$ particle and all particles are within observation range, we define the consensus-based policy sequence formally as:
\begin{align*}
    a_{0} &= p_{ij,b}\sphericalangle p_{ji,b} \ \forall i,j\in \mathcal{A}\\
    a_{1} &= \anchor(A_j,p_{ij,b+1} \ \forall j\in \mathcal{A})\\
    a_s &= a_0\otimes a_1 
\end{align*}
such that the frame-policy update is:
\begin{equation}
    f_{b+1} = f_{b}\otimes a_s
    \label{eq:strategy}
\end{equation}
where $b \in \mathcal{B} \setminus \{N_b\}$ and $\otimes$ is a modal product. This strategy is used until all robots communicate at which time the particles are reset to the robots' poses and dynamics are updated.

\textbf{\textit{Example:}} Consider $f_0,f'_0,f_1$ in the Kripke model shown in Fig.~\ref{fig:example1} where the true world is denoted by the black vertex, edges represent accessibility relation ($R$), and new propositions ($V$) are denoted in each worlds' sets. From the perspective of robot $A_1$, each robot initially tracks particles $p_{11,1}$, $p_{12,1}$ and $p_{13,1}$. All robots propagates four possible belief states for each teammate and four empathy states of itself. We denote robot $A_1$'s knowledge and belief of his own state as $K_1$ \track$(A_1,p_{11,1})$. Similarly, robot $A_1$'s knowledge and belief about his teammate's state is represented as $B_1 \track(A_j,p_{1j,1}), j\in\{2,3\}$ and empathy from robot $A_1$'s perspective is shown as $B_1B_j\track(A_1,p_{j1,1}), j\in\{2,3\}$ and can be read as ``robot $A_1$ believes that robot $j$ believes that robot $A_1$ is tracking particle 1". We denote the initial common belief propositions for all three agents as $\anchor(A_i,p_{i1,1},p_{i2,1},p_{i3,1}), \ i\in{1,2,3}$. 

\begin{figure}[t]
    \hspace{-3pt}\includegraphics[width = 0.46\textwidth]{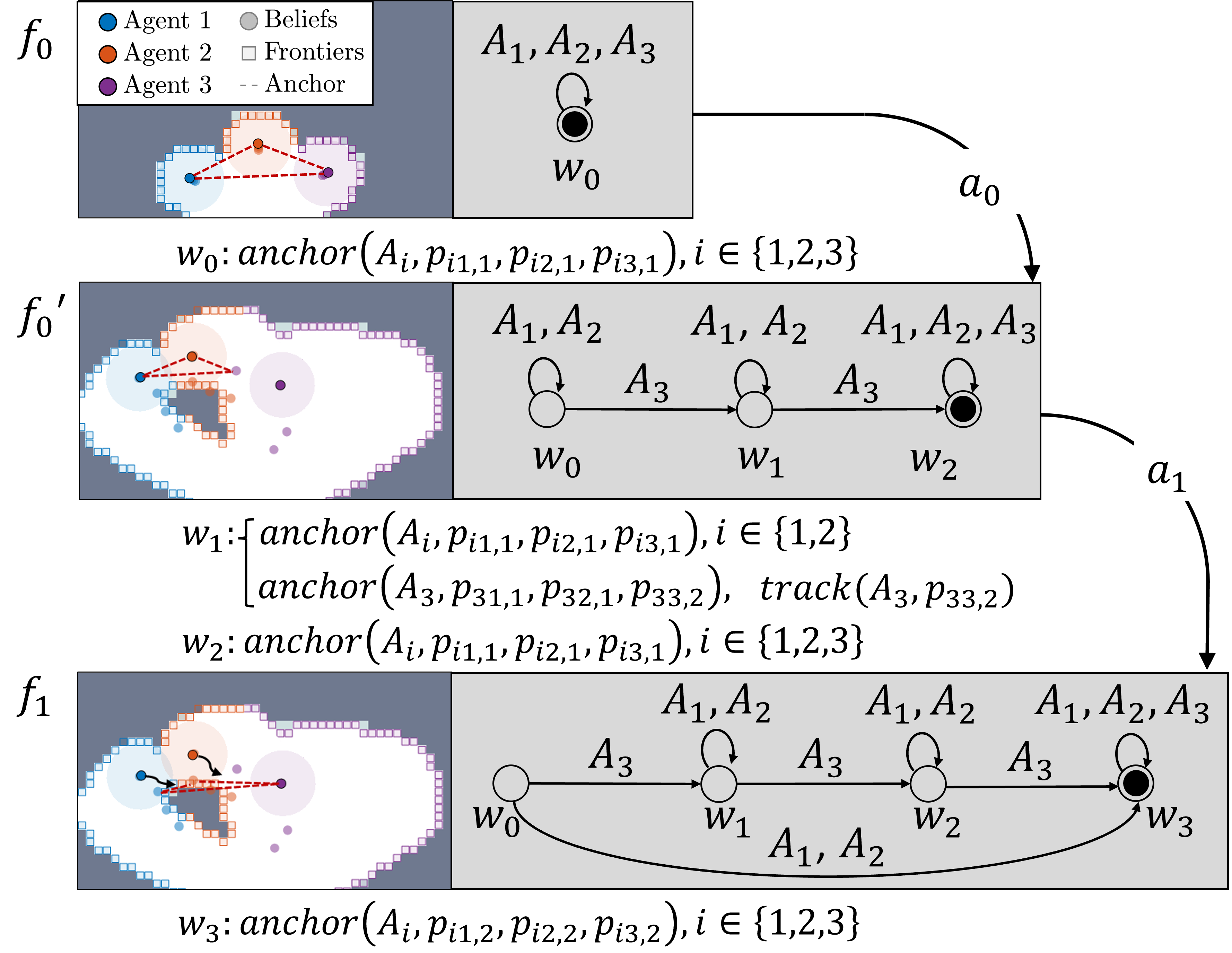}
    \vspace{-5pt}
    \caption{Example scenario where $A_3$ experiences a failure. The Kripke models are shown in the gray boxes with frame transitions denoted as $a_n$.}
    \label{fig:example1}
    \vspace{-15pt}
\end{figure}

After disconnection (as shown in $f'_0$ in Fig.~\ref{fig:example1}), the robots can no longer assume to have knowledge about the true state of the system. In this scenario, $A_3$ experiences a failure such that $\track(A_3,p_{33,2})$ and updates its common belief ($w_1$). However, $A_3$ subsequently reasons that $A_1$ and $A_2$ hold a false belief that $A_3$ is still tracking the first particle. Thus, $A_3$ reverts its common belief in $w_2$ to mirror the initial proposition such that $\anchor(A_i,p_{i1,1},p_{i2,1},p_{i3,1})$.

When $A_1\sphericalangle p_{13,1}$ in $f'_0$, $A_1$ and $A_2$ are are connected and so $A_1$ relays that $A_1\sphericalangle p_{13,1}\land A_1\sphericalangle\neg A_3$, reasoning that $A_3$ does not know that both agents are still tracking their respective first particle. $A_3$ also knows that all three first particles are within communication range, but cannot observe any additional information. Thus, all three agents update their common belief such that $\anchor(A_i,p_{i1,2},p_{i2,2},p_{i3,2})$ and all particles begin converging to the updated common belief set.

\subsection{Obstacle Avoidance \& Task Completion} \label{sec:avoidance_and_taskCompletion}
Finally, we consider the last two forces in \eqref{eq:taskTotalForce}. To avoid obstacles, force $F^3_{ij,b}$ is formulated as: 
\begin{equation}
\vspace{-5pt}
    F^3_{ij,b} = \dfrac{1}{|\mathcal{O}|}\sum_{o\in \mathcal{O}}\dfrac{s_o-p_{ij,b}}{||s_o-p_{ij,b}||^3}\label{eq:avoidobstacles}
\end{equation}
where $s_o \in \mathcal{C}$ is the coordinate for an obstacle $o\in\mathcal{O}$ in the environment. We note that $\mathcal{O}$ is a set of the commonly known obstacles for all agents. For example, if a robot individually encounters an obstacle, but has not communicated its location to all teammates, the particles' motions are not affected by this new obstacle which is unknown to other robots. 

For attraction to tasks, the force $F^4_{ij,b}$ is only active when all agents are connected to propagate particles towards any commonly known tasks. These tasks are centrally auctioned and each robot submits a bid based on estimated traversal time to the task. Assigned tasks placed in a queue set $\mathcal{Q}_i$, and distributed to robot $i$'s particle task queues, $\mathcal{Q}_{ij,b}$. Attraction to the first task, $\mathcal{Q}_{ij,b}[1]$, in a particle's queue is formulated as:
\begin{equation}
    F^4_{ij,b} = \mathcal{Q}_{ij,b}[1] - p_{ij,b}
\end{equation}
where the coefficients for $F^1_{ij,b}$ and $F^4_{ij,b}$ depend on the particle's queue $\mathcal{Q}_i$, such that:
\begin{equation}
    \left\{ \begin{array}{rl}
    \beta_1 = 0 \text{ and } \beta_4>0, & \text{if } \mathcal{Q}_i \neq \emptyset, \\[2pt]
    \beta_1 > 0 \text{ and } \beta_4 = 0, & \text{otherwise}.
    \end{array} \right.
\end{equation}
Once the task queue is empty, the particle force for tasks is set to zero and coverage resumes.

\subsection{Particle Tracking}
After particle propagation has occurred, a robot must predict and track its empathy particle considering the possibility of a new obstacle or discovering a task while disconnected. Though any constrained tracking algorithm can be employed, we use a nonlinear receding horizon controller (RHC) to minimize the distance to the particle while maintaining a radius from any obstacle \cite{mayne1988receding}. Additionally, if the team is disconnected and a robot discovers a task within its observing radius $r_o$, the tasks are re-indexed by monotonically increasing cost and placed in the set $\mathcal{Q}_i$. The controller's goal is updated to the first queued task location. Once the robot has traveled within the completion radius, $r_t$, the robot continues tracking its respective particle and the task is removed from the robot's queue. An example for both obstacle avoidance and task discovery are shown in Fig.~\ref{fig:deviations}. 
\begin{figure}[ht!]
\vspace{-2pt}
 \subfigure{\includegraphics[width = 0.07\textwidth]{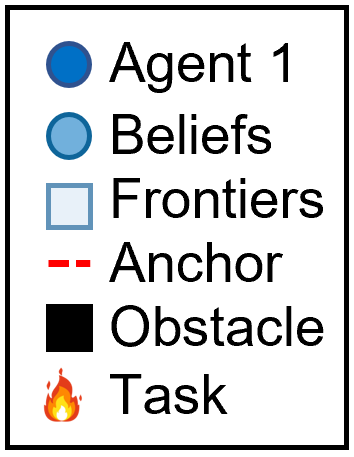}}
    \renewcommand{\thesubfigure}{(a)}
        \hspace{5pt}
    \subfigure[]{
    \includegraphics[width = 0.17\textwidth]{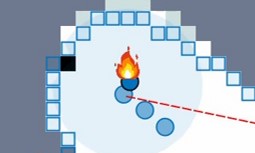}}
    \renewcommand{\thesubfigure}{(b)}
    \hspace{8pt}
    \subfigure[]
    {\includegraphics[width = 0.17\textwidth]{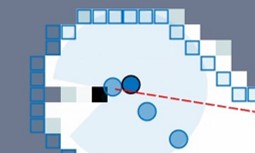}}
\vspace{-5pt}
\caption{Examples of necessary deviation. In (a) the robot spots a task and updates its RHC goal. In (b) the robot must avoid a discovered obstacle while minimizing distance to the particle.}
\label{fig:deviations}
\vspace{-10pt}
\end{figure}

%% file: 6Results.tex
\section{Simulations} \label{sec:sims}
In this section, we provide results and comparisons from MATLAB simulations with our approach implemented on a two robot team. Simulations were performed on 15 random $50\text{m}\times 50\text{m}$ environments with 5-15 initially unknown obstacles and a maximum of seven tasks. The robots start by assuming that the environment has no obstacles and do not know the location of the tasks. We compare the results between: 1) no failures, 2) one failure by one robot, and 3) one failure by each robot at random times.
\begin{figure}[bh!]
    \centering
    \includegraphics[width = 0.48\textwidth]{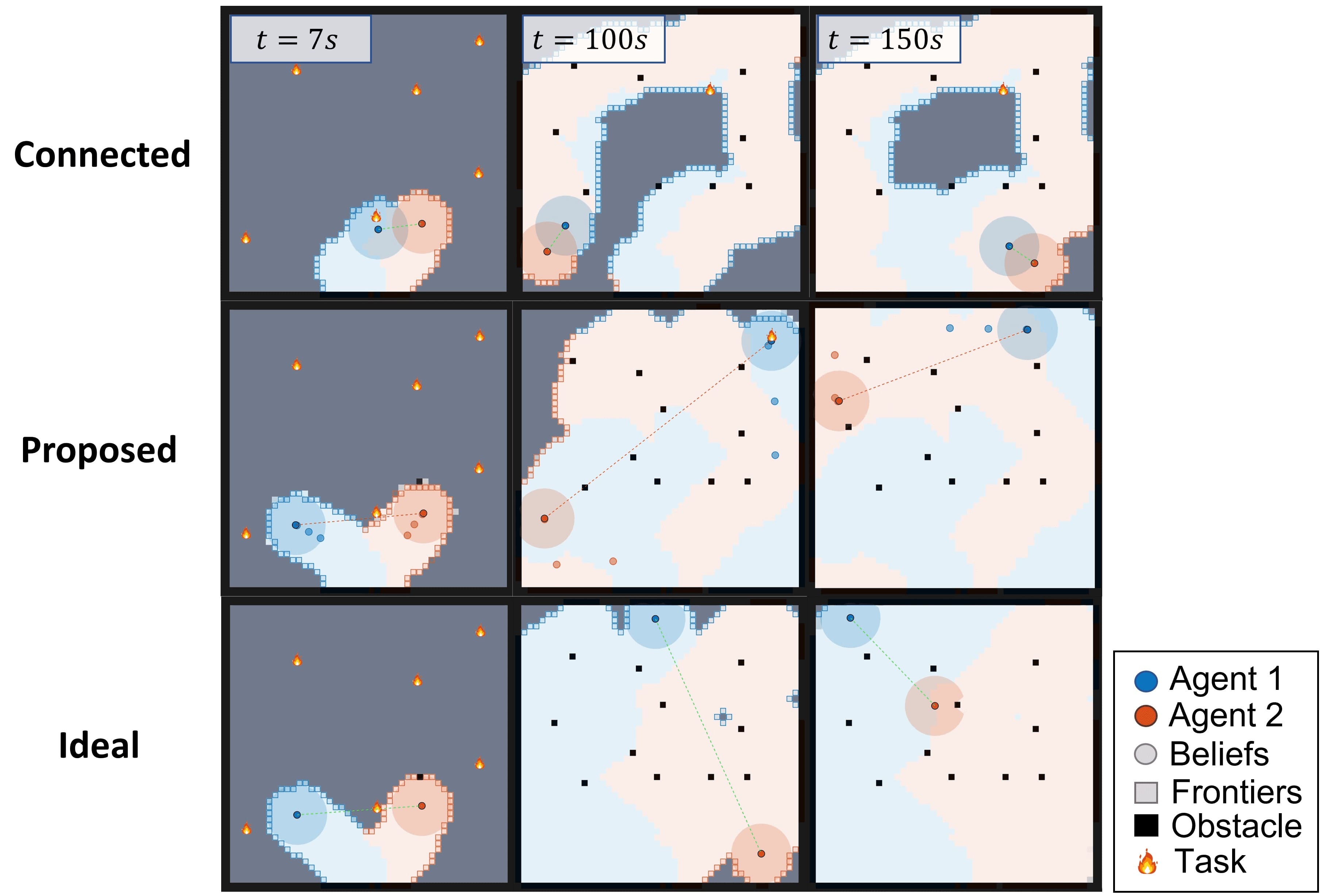}
    \vspace{-5pt}
    \caption{Snapshots of example simulations. The robots experience a failure each at random times with 11 unknown obstacles and 7 tasks.}
    \label{fig:picComparison}
    \vspace{-0pt}
\end{figure}

\begin{figure}[bh!]
    \vspace{-5pt}
    \includegraphics[width = 0.42\textwidth]{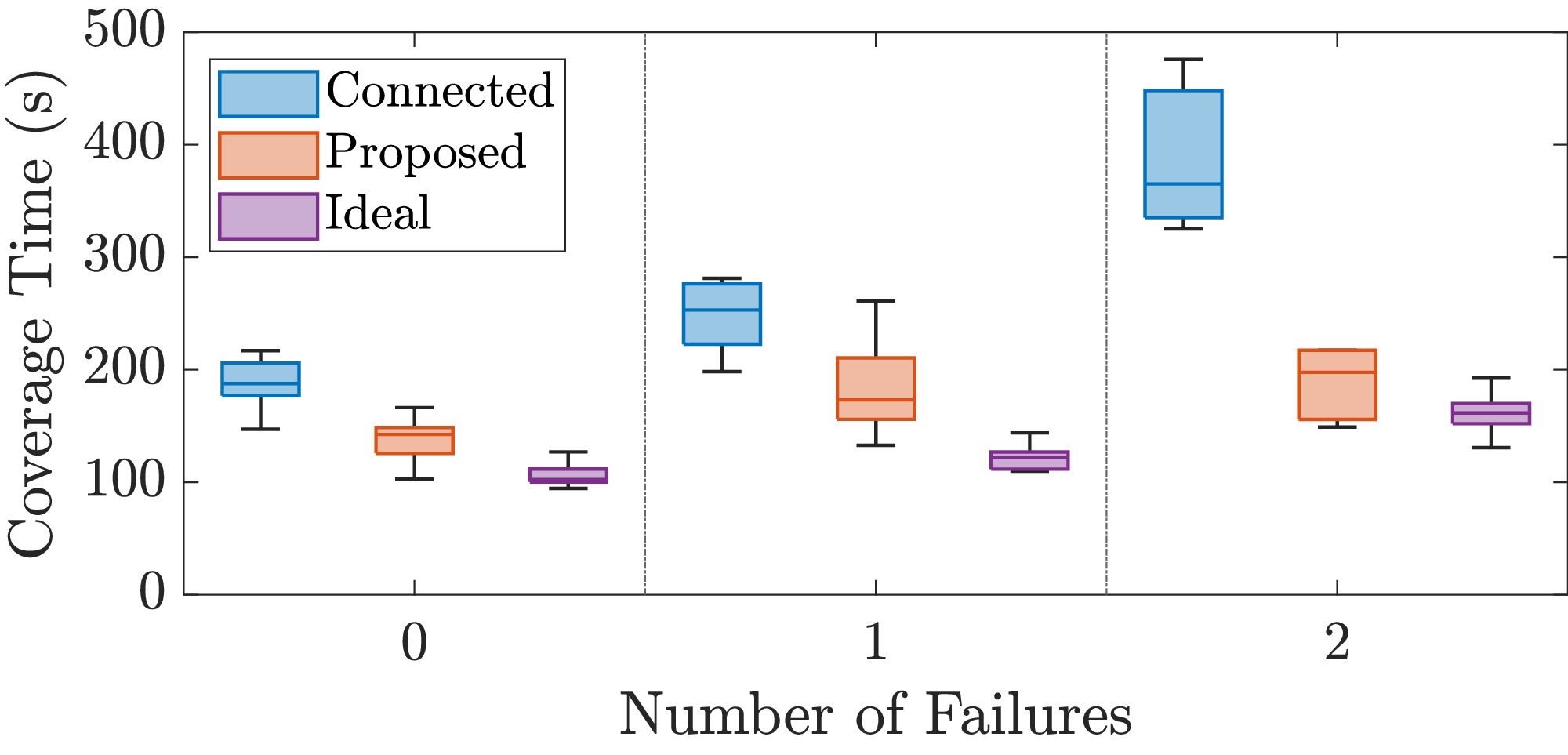}
    \vspace{-5pt}
    \caption{Comparison between methods given 0, 1, and 2 failures.}
    \label{fig:comparison}
    \vspace{-5pt}
\end{figure}

\begin{figure*}[t]
\vspace{-2pt}
    \subfigure[]{
    \includegraphics[width = 0.18\textwidth,trim = {0cm 0cm 0cm 0},clip]{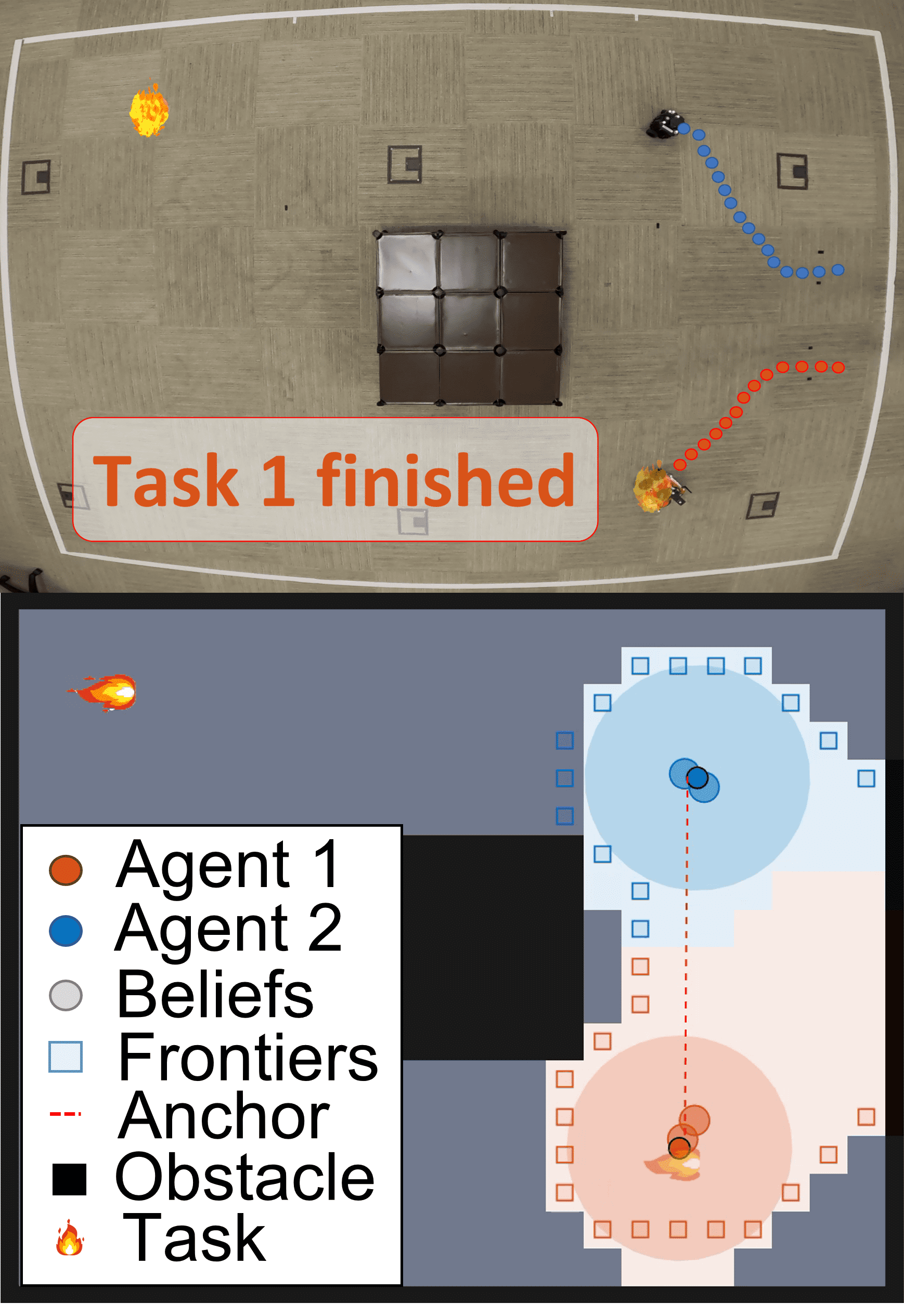}
    \label{fig:exa}
    }%
    \subfigure[]{
    \includegraphics[width = 0.18\textwidth,trim = {0cm 0cm 0cm 0},clip]{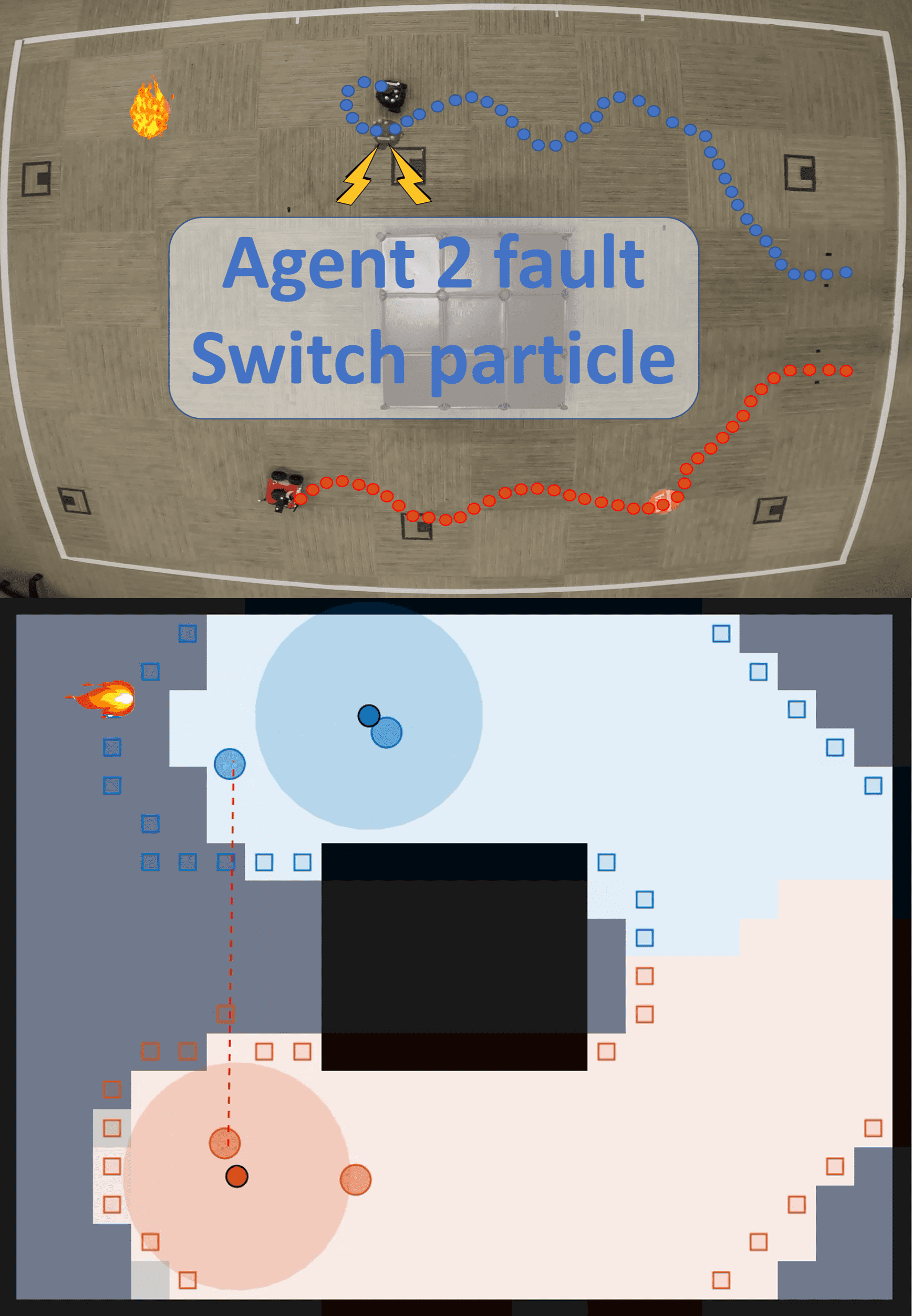}
    \label{fig:exb}
    }%
    \subfigure[]{
    \includegraphics[width = 0.18\textwidth,trim = {0cm 0cm 0cm 0},clip]{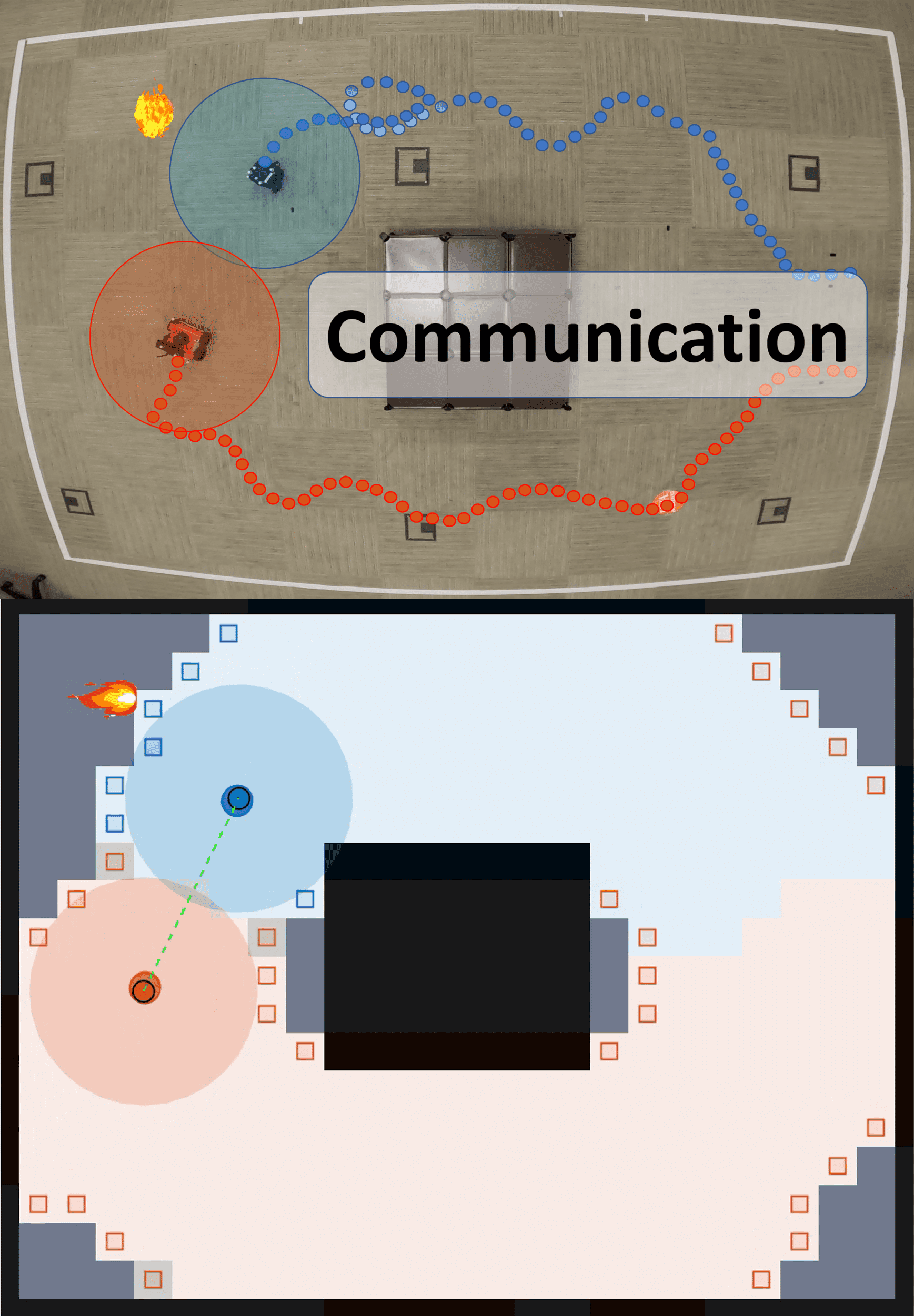}
    \label{fig:exc}
    }
    \subfigure[]{
    \includegraphics[width = 0.18\textwidth,trim = {0cm 0cm 0cm 0},clip]{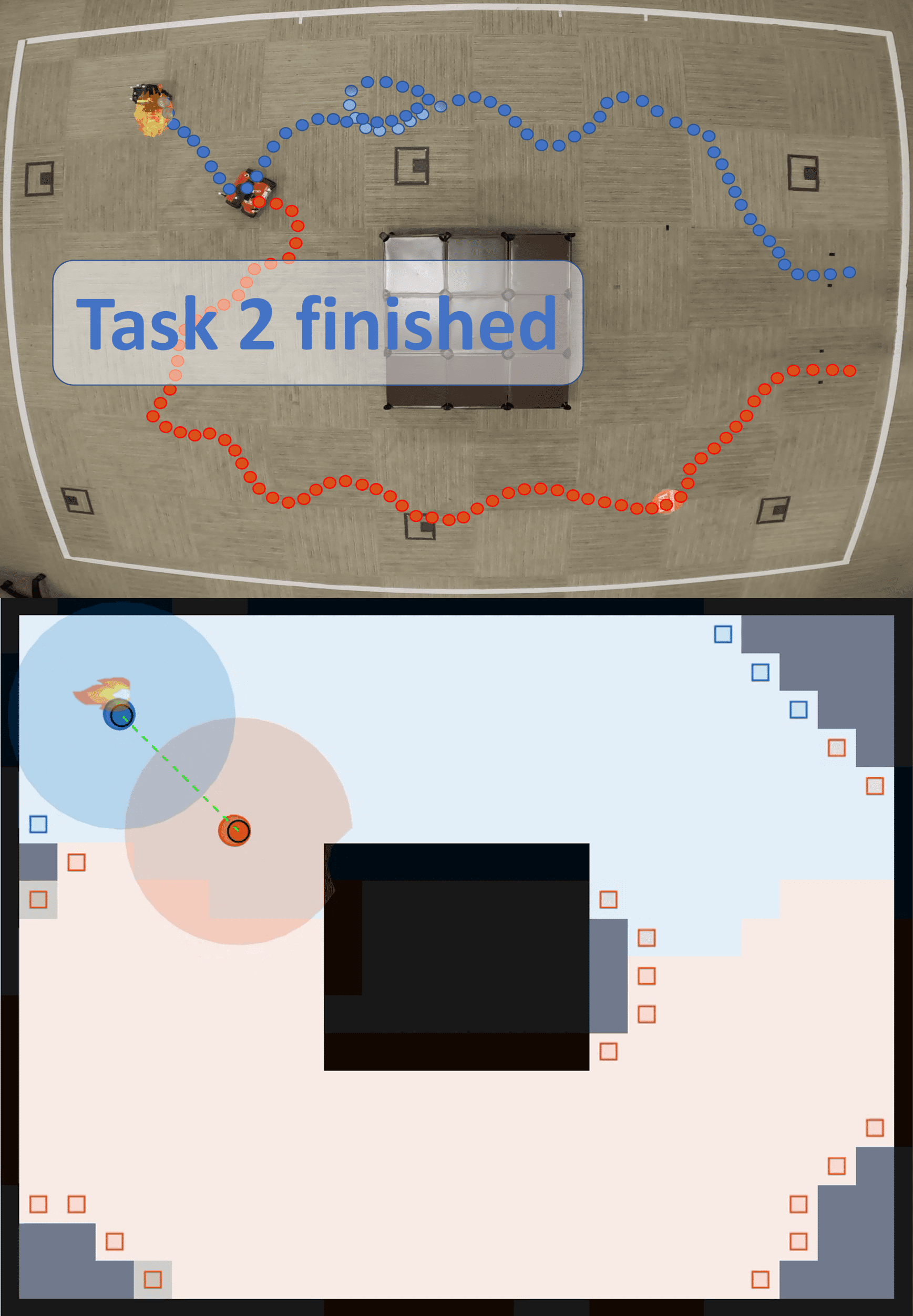}
    \label{fig:exd}
    }
    \subfigure[]{
    \includegraphics[width = 0.18\textwidth,trim = {0cm 0cm 0cm 0},clip]{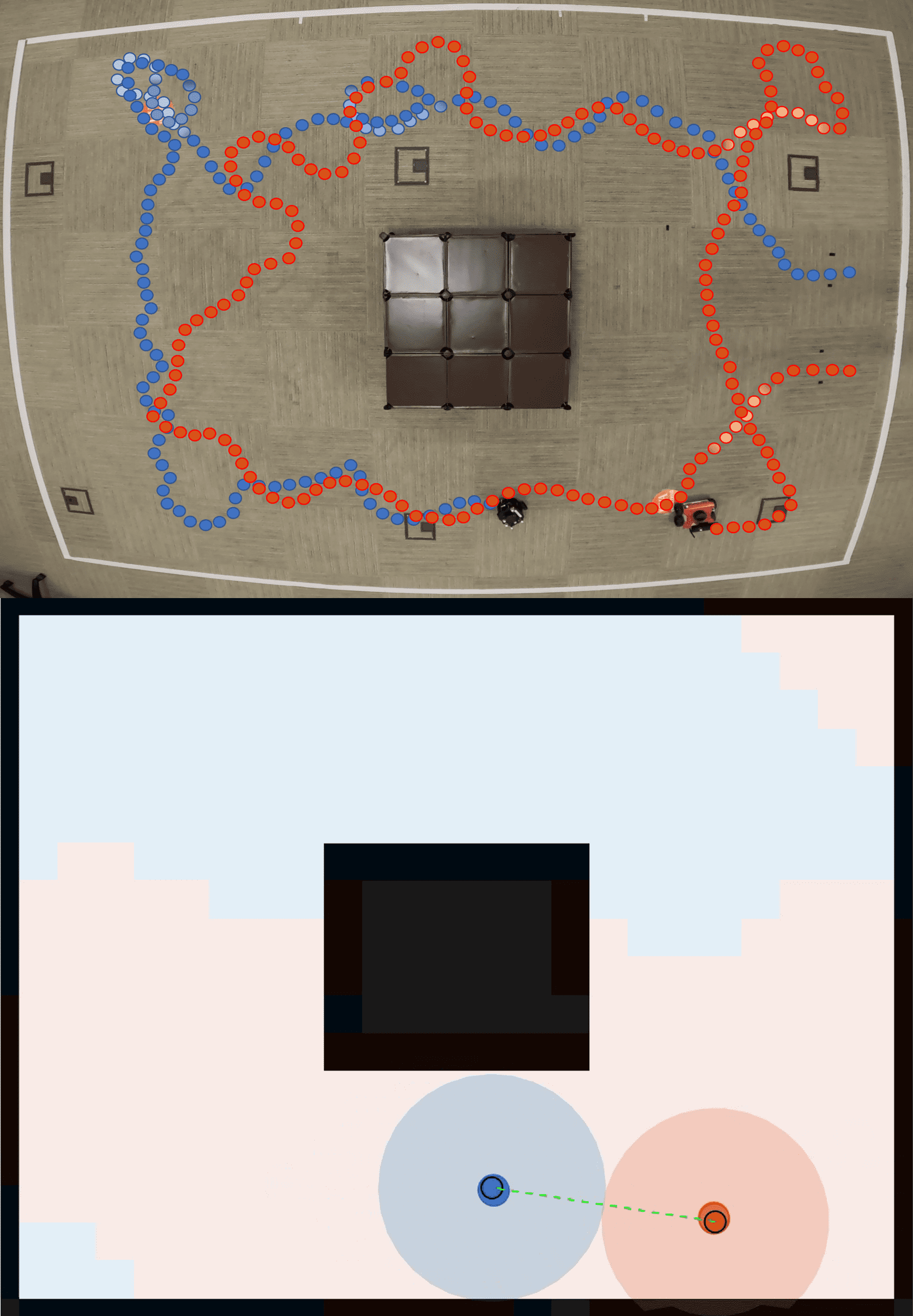}
    \label{fig:exe}
    }
\vspace{-5pt}
\caption{Snapshots and results of an experimental case study.}
\label{fig:expFig}
\vspace{-8pt}
\end{figure*}

Each robot propagates three particles traveling at 2m/s, 1m/s and 0.5m/s. A failure can cause the robot to track either the second or third particle. The particles propagate according to these failure velocities and the maximum communication range is 10m from the center of the robot. 

The proposed approach is compared against two other methods. The first method applies a constant connectivity constraint, not allowing agents to travel outside of a 10m communication range. The second method assumes ideal conditions, where robots can communicate across the entire environment. Both methods use an artificial potential field technique for controlling the robots towards uncovered regions and away from obstacles. In all methods, the initial maximum velocity is 2m/s, the simulated LiDAR range is 5m, and the robots' motion is modeled using unicycle dynamics. Fig.~\ref{fig:picComparison} shows an example at various time steps for the three methods, displaying the coverage disparity of the connected method versus our proposed method and showing the similarities between our method and the ideal method.

As shown in Fig.~\ref{fig:comparison}, the proposed method outperforms the fully connected method in all scenarios. Additionally, the median coverage time for the proposed method is similar to the ideal method, even with the communication limitation. 
\section{Experiments}
The proposed approach was also validated through laboratory experiments with a two-robot team. The team consists of a Husarion ROSbot 2.0 UGV and a Turtlebot3 Burger UGV using a Vicon motion capture system. 
The two-robot experiments effectively demonstrate all parts of the proposed approach, including intentional disconnections, searching, and rendezvousing behaviors. In all experiments, the UGVs start within communication range and are tasked to cover the environment and complete any discovered tasks.

Experiments were performed in a $4$m$\times 5.5$m space containing convex obstacles considering, as a proof of concept, a sensing and communication range for each robot of $1$m. Displayed in Fig.~\ref{fig:expFig} are the results from the two-robot experiment in which the vehicles are required to search for and complete two tasks unknown a priori. The columns of Fig.~\ref{fig:expFig} correspond to different instances within the experiment, and each row from top to bottom shows snapshots of the robots at different times throughout the experiment and the current map of the environment covered by the team. 
In Fig.~\ref{fig:expFig}(a), the UGVs start to cover the map in search of tasks and disconnect. Robot 1 (ROSbot) finds a task and completes it. In Fig.~\ref{fig:expFig}(b), robot 2 (Turtlebot) experiences a fault and begins following the second empathy particle. In Fig.~\ref{fig:expFig}(c-d) the robots connect, share fault information, and bid on the discovered task. Robot 1 receives a larger share of the frontier as a result of Robot 2's failure and Robot 1 is assigned the task based on proximity. Finally, once all tasks are completed and no frontiers remain, the agents rendezvous and the final results are shown in Fig.~\ref{fig:expFig}(e). Final coverage over time for each robot is shown in Fig.~\ref{fig:expResults}. More lab experiments with two-robots are included in the supplementary material.

\begin{figure}[t]
    \centering
    \includegraphics[width = 0.45\textwidth]{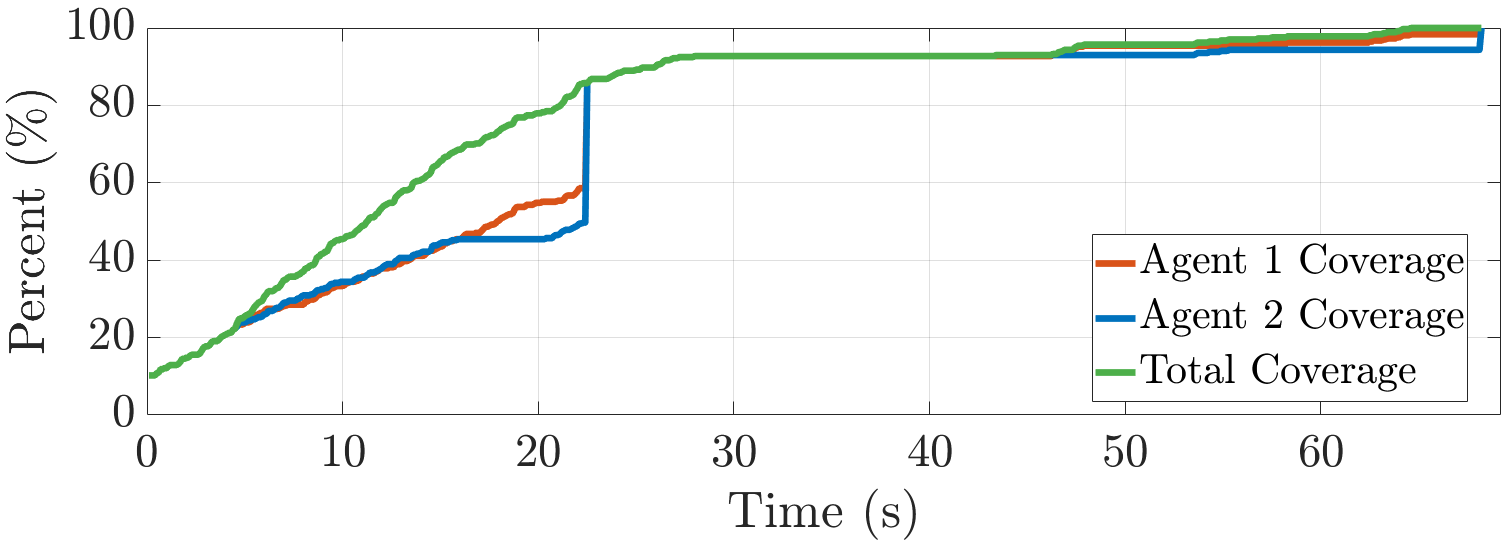}
    \vspace{-10pt}
    \caption{Graphical coverage over time for the experimental case study.}
    \label{fig:expResults}
    \vspace{-18pt}
\end{figure}

\label{sec:exp}

%% file: 7Conclusion.tex
\section{Conclusion \& Future Work} \label{sec:concs}
In this work, we have presented a novel framework for multi-robot systems to use epistemic planning to propagate and behave according to a set of beliefs. The proposed method promotes disconnection using frontier-based exploration and includes a dynamic rendezvous approach to reconnect and share data among the multi-robot system. The extensive simulations and experiment results show the validity, applicability, generality of the proposed method. 
Using this framework, we also demonstrate improved  task completion and coverage time of partially known environments with respect to standard coverage methods. From here, future theoretical work includes addressing the challenges of dynamic task lengths and optimal strategies for intentional information sharing in complex or unknown environments. Further modeling of epistemic planning using epistemic MDPs to reach a more informed consensus for faster information dissemination is also on our agenda. Additionally, we plan to apply this framework in outdoor experiments using robots equipped with short-range communication devices.
\section{Acknowledgement}
This work is based on research sponsored by Northrop Grumman through the University Basic Research Program and DARPA under Contract No. FA8750-18-C-0090.